\documentclass{llncs}

\usepackage{amsmath}
\usepackage{amssymb}
\usepackage{color}
\usepackage{listings}
\usepackage{tikz} \usetikzlibrary{arrows,automata,calc,decorations.pathmorphing}
\usepackage{hyperref}

\newcommand{\QED}{\hfill $\Box$}

\newcommand{\Among}{\Constraint{Among}}
\newcommand{\Arc}[2]{#1~\Set{#2}}                            % #1=symbol, #2=update
\newcommand{\Automaton}{\Constraint{Automaton}}
\newcommand{\cAutomaton}{\Constraint{cAutomaton}}
\newcommand{\cAutomatonAtLeast}{\Constraint{cAutomatonAtLeast}}
\newcommand{\cAutomatonAtMost}{\Constraint{cAutomatonAtMost}}
\newcommand{\Cardinality}[1]{\lvert#1\rvert}
\newcommand{\Constraint}[1]{\textsc{#1}}                     % #1 = constraint name
\newcommand{\CostRegular}{\Constraint{CostRegular}}
\newcommand{\DFA}{\mathcal{A}}
\newcommand{\Domain}[1]{\mathrm{dom}(#1)}                    % #1 = variable name
\newcommand{\EmptyString}{\epsilon}
\newcommand{\Eq}{$\Set{k \gets 0}$}
\newcommand{\Iff}{\Leftrightarrow}
\newcommand{\Inflexion}{\Constraint{Inflexion}}
\newcommand{\Initial}{q_0}
\newcommand{\NumberWord}{\Constraint{NumberWord}}
\newcommand{\Regular}{\Constraint{Regular}}
\newcommand{\Sequence}[1]{[#1]}
\newcommand{\Set}[1]{\{#1\}}                                 % #1 = set elements
\newcommand{\SetComp}[2]{\Set{#1 \mid #2}}
\newcommand{\Tuple}[1]{\langle#1\rangle}                     % #1 = tuple elements

\newcommand{\Cost}[1]{\mathcal{C}(#1)}                       % #1=transition/path
\newcommand{\Inc}{\mathit{inc}}
\newcommand{\Next}[2]{\TransitionFct(#1,#2)}                 % #1=state, #2=symbol
\newcommand{\NextN}[2]{\TransitionFct_\mathbb{N}(#1,#2)}     % #1=state, #2=symbol
\newcommand{\NextQ}[2]{\TransitionFct_\mathrm{Q}(#1,#2)}     % #1=state, #2=symbol
\newcommand{\Path}[3]{#1  \overset{#2}{\rightsquigarrow} #3}  % #1=from, #2=string, #3=to
\newcommand{\Transition}[3]{#1 \overset{#2}{\rightarrow} #3}  % #1=from, #2=symbol, #3=to
\newcommand{\TransitionFct}{\delta}

\newcommand{\PreMin}[1]{\underline{\mathrm{QCF}}(#1)}  % #1=index,symbol
\newcommand{\PreMax}[1]{\overline{\mathrm{QCF}}(#1)}   % #1=index,symbol
\newcommand{\SufMin}[1]{\underline{\mathrm{QCB}}(#1)}  % #1=index,symbol
\newcommand{\SufMax}[1]{\overline{\mathrm{QCB}}(#1)}   % #1=index,symbol
\newcommand{\Min}[1]{\underline{m}(#1)}                % #1=index,symbol

\DeclareMathOperator*{\trimMin}{trimMin}
\newcommand{\TrimMin}[2]{\trimMin\limits_{\substack{#1}}(#2)}     % #1=quant, #2=set

\title{Propagating Regular Counting Constraints}

\author{
  Nicolas Beldiceanu \inst{1} \and
  Pierre Flener \inst{2} \and \\
  Justin Pearson \inst{2} \and 
  Pascal Van Hentenryck \inst{3}
}

\institute{TASC team (CNRS/INRIA), Mines de Nantes, 44307 Nantes, France \\
  \url{Nicolas.Beldiceanu@mines-nantes.fr}
  \and
  Uppsala University, Dept of Information Technology,
  751 05 Uppsala, Sweden \\
  \url{Pierre.Flener@it.uu.se, Justin.Pearson@it.uu.se}
  \and
  Optimization Research Group, NICTA, and The University of Melbourne,
  Australia \\
  \url{pvh@nicta.com.au}
}

\begin{document}

\maketitle

\begin{abstract}
  Constraints over finite sequences of variables are ubiquitous in
  sequencing and timetabling.  Moreover, the wide variety of such
  constraints in practical applications led to general modelling
  techniques and generic propagation algorithms, often based on
  deterministic finite automata (DFA) and their extensions.  We
  consider counter-DFAs (cDFA), which provide concise models for
  regular counting constraints, that is constraints over the number of
  times a regular-language pattern occurs in a sequence.  We show how
  to enforce domain consistency in polynomial time for \emph{atmost}
  and \emph{atleast} regular counting constraints based on the
  frequent case of a cDFA with only accepting states and a single
  counter that can be incremented by transitions.  We also prove that
  the satisfaction of \emph{exact} regular counting constraints is
  NP-hard and indicate that an incomplete algorithm for \emph{exact}
  regular counting constraints is faster and provides more pruning
  than the existing propagator from~\cite{Beldiceanu:automata}.
  Regular counting constraints are closely related to the
  $\CostRegular$ constraint but contribute both a natural abstraction
  and some computational advantages.
\end{abstract}

\section{Introduction}

Constraints over finite sequences of variables arise in many
sequencing and timetabling applications, and the last decade has
witnessed significant research on how to model and propagate, in a
generic way, idiosyncratic constraints that are often featured in
these applications.  The resulting modelling techniques are often
based on formal languages and, in particular, deterministic finite
automata (DFA).  Indeed, DFAs are a convenient tool to model a wide
variety of constraints, and their associated propagation algorithms
can enforce domain consistency in polynomial
time~\cite{Beldiceanu:automata,Pesant:regular}.

This paper is concerned with the concept of counter-DFA (cDFA), an
extension of DFAs proposed in~\cite{Beldiceanu:automata}, and uses it
to model regular counting constraints, that is constraints on the
number of regular-language patterns occurring in a sequence of
variables.  cDFAs typically result in more concise and natural
encodings of regular counting constraints compared to DFAs, but it is
unknown if they admit efficient propagators enforcing domain
consistency.  This paper originated as an attempt to settle this
question.  We consider the subset of cDFAs satisfying two conditions:
(1)~all their states are accepting, and (2)~they manipulate a single
counter that can be incremented by transitions.  These conditions are
satisfied for many regular counting constraints and offer a good
compromise between expressiveness and efficiency.

Our main contribution is to show that, for such a counter automaton
$\DFA$, it is possible to enforce domain consistency efficiently on
\emph{atmost} and \emph{atleast} regular counting constraints.
Constraint $\cAutomatonAtMost(N,X,\DFA)$ holds if the counter of
$\DFA$ is at most $N$ after $\DFA$ has consumed sequence $X$.
Constraint $\cAutomatonAtLeast(N,X,\DFA)$ is defined similarly.  We
also prove the NP-hardness of satisfiability testing for constraint
$\cAutomaton(N,X,\DFA)$, which holds if the counter of $\DFA$ is
exactly $N$ after $\DFA$ has consumed $X$.  Compared to the
$\CostRegular$ constraint~\cite{CostRegular}, as generalised for the
\emph{Choco} solver~\cite{CHOCO}, our contribution is a propagator for
exact regular counting that uses asymptotically less space (for its
internal datastructures) and yet propagates more on the variables of
$X$.  Furthermore, our propagators for \emph{atmost} and
\emph{atleast} regular counting achieve domain consistency on the
counter variable $N$ (and~$X$) in the same asymptotic time as the
$\CostRegular$ propagator achieves only bounds consistency on $N$ (but
also domain consistency on~$X$).

The rest of the paper is organised as follows.
Section~\ref{section:problem} defines cDFAs, regular counting
constraints, and the class of cDFAs considered.
Section~\ref{section:propagator} gives the propagator, its complexity,
and its evaluation.  Section~\ref{section:conclusion} concludes the
paper and discusses related work.

\section{Background}
\label{section:problem}

\subsection{Deterministic Finite Counter Automata} 

Recall that a \emph{deterministic finite automaton} (DFA) is a tuple
$\Tuple{Q,\Sigma,\TransitionFct,\Initial,F}$, where $Q$ is the set of
states, $\Sigma$ is the alphabet, $\TransitionFct \colon Q \times
\Sigma \to Q$ is the transition function, $\Initial \in Q$ is the
start state, and $F \subseteq Q$ is the set of accepting states.

This paper considers a subclass of counter-DFAs in which all states
are accepting and only one counter is used.  The counter is
initialised to $0$ and increases by a given natural number at every
transition.  Such an automaton accepts every string and assigns a
value to its counter.
More formally, a \emph{counter-DFA} (cDFA) is here specified as a
tuple $\Tuple{Q,\Sigma,\TransitionFct,\Initial,F}$, where $Q$,
$\Sigma$, $\Initial$, and $F$ are as in a DFA except that $F=Q$ and
the DFA transition function $\TransitionFct$ is extended to the
signature $Q \times \Sigma \to Q \times \mathbb{N}$, so that
$\Next{q}{\ell} = \Tuple{r,\Inc}$ indicates that $r$ is \emph{the}
successor state of state $q$ upon reading alphabet symbol $\ell$ and
the counter must be incremented by $\Inc$.
We also define two projections of this extended transition function:
if $\Next{q}{\ell} = \Tuple{r,\Inc}$, then $\NextQ{q}{\ell} = r$ and
$\NextN{q}{\ell} = \Inc$.
Given $\Next{q}{\ell} = \Tuple{r,\Inc}$, we denote by
$\Cost{\Transition{q}{\ell}{r}}$ the counter increase $\Inc$ of
transition $\Transition{q}{\ell}{r}$ from state~$q$ to state~$r$ upon
consuming symbol~$\ell$.  Similarly, we denote by
$\Cost{\Path{q}{\sigma}{r}}$ the counter increase of a path
$\Path{q}{\sigma}{r}$ from state~$q$ to state~$r$ upon consuming a
(possibly empty) string~$\sigma$.

\begin{example}
  Consider the automaton $\mathcal{AAB}$ in Figure~\ref{fig:aab}.  It
  represents a cDFA with state set $Q = \Set{\epsilon, a,
    \mathit{aa}}$ and alphabet $\Sigma = \Set{\mathrm{a},
    \mathrm{b}}$.  The transition function $\TransitionFct$ is given
  by the labelled arcs between states, and the start state is
  $\Initial = \epsilon$ (indicated by an arc coming from no state; we
  often denote the start state by~$\EmptyString$, because it can be
  reached by consuming the empty string~$\EmptyString$).  Since the
  final states $F$ are all the states in $Q$, this automaton
  recognises every string over $\Set{\mathrm{a}, \mathrm{b}}$ and is
  thus by itself not very interesting.  However, the cDFA features a
  counter $k$ that is initialised to $0$ at the start state,
  incremented by $1$ on the transition from state $\mathit{aa}$ to
  state $\epsilon$ upon reading symbol `$\mathrm{b}$', and incremented
  by $0$ on all other transitions.  As a result, the final value of
  $k$ is the number of occurrences of the word ``$\mathrm{aab}$''
  within the string.  \QED
\end{example}

\begin{figure}[t]
  \centering
  \begin{tikzpicture}[->,>=stealth',shorten >=1pt,auto,node distance=20mm,semithick]
    \node[initial,accepting,initial text=\Eq,initial distance=5mm,state] (A) {$\epsilon$};
    \node[state,accepting] (B) [right of=A] {$a$};
    \node[state,accepting] (C) [right of=B] {$\mathit{aa}$};
    \path
    (A)
    edge[in=-150,out=-120,loop]	node[left]{b}			(A)
    edge			node{a}				(B)
    (B)
    edge[bend right]		node[above]{b}			(A)
    edge			node{a}				(C)
    (C)
    edge[in=-60,out=-30,loop]	node[right]{a}			(C)
    edge[bend left] 		node{\Arc{b}{$k \gets k+1$}}	(A);
  \end{tikzpicture}
  \caption{Counter-DFA $\mathcal{AAB}$ for the constraint
    $\NumberWord(N,X,\text{``aab''})$}
  \label{fig:aab}
\end{figure}
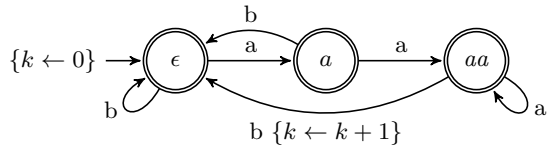

\subsection{Regular Counting Constraints}

A \emph{regular counting constraint} is defined as a constraint
that can be modelled by a cDFA.  The $\cAutomaton(N,X,\DFA)$
constraint holds if the value of variable~$N$, called the
\emph{counter variable}, is equal to the final value of the counter
after cDFA~$\DFA$ has consumed the values of the entire sequence $X$
of variables.  Consider the constraint $\NumberWord(N,X,w)$, which
holds if $N$ is the number of occurrences of the non-empty word $w$ in
the sequence $X$ of variables.  Constraint
$\NumberWord(N,X,\text{``aab''})$ can be modelled by the
$\cAutomaton(N,X,\mathcal{AAB})$ constraint with the automaton
$\mathcal{AAB}$ specified in Figure~\ref{fig:aab}.

\subsection{Signature Constraints}

A constraint on a sequence $X$ of variables can sometimes be modelled
with the help of a DFA or cDFA that operates not on $X$, but on a
sequence of \emph{signature variables} that functionally depend via
\emph{signature constraints} on a sliding window of variables within
$X$~\cite{Beldiceanu:automata}.

For example, the $\Among(N,X,\mathcal{V})$
constraint~\cite{Beldiceanu:globals} requires $N$ to be the number of
variables in the sequence $X$ that are assigned a value from the given
set~$\mathcal{V}$.  With signature constraints $x_i \in \mathcal{V}
\Iff s_i=1$ and $x_i \notin \mathcal{V} \Iff s_i=0$ (with $x_i \in
X$), we obtain a sequence of $\Cardinality{X}$ signature variables
$s_i$ that can be used in a cDFA that counts the number of occurrences
of value~$1$ in that sequence.  Rather than labelling the transitions
of such a cDFA with \emph{values} of the domain of the signature
variables (the set $\Set{0,1}$ here), we label them with the
corresponding \emph{conditions} of the signature constraints, as in
the cDFA $\mathcal{AMONG}$ given in Figure~\ref{fig:among}.  Note that
the choice of $\Among$ is purely pedagogical: we do not argue that
this is the best way to model and propagate this constraint.

\begin{figure}[t]
  \centering
  \begin{tikzpicture}[->,>=stealth',shorten >=1pt,auto,node distance=2cm,semithick]
    \node[initial,accepting,initial text=\Eq,initial distance=5mm,state] (A) {$i$};
    \path
    (A)
    edge[in=40,out=10,loop]	node[right]{$x_i \notin \mathcal{V}$}	  		(A)
    edge[in=-40,out=-10,loop]	node[right]{\Arc{$x_i \in \mathcal{V}$}{$k \gets k+1$}} (A);
  \end{tikzpicture}
  \caption{Counter-DFA $\mathcal{AMONG}$ for the constraint
    $\Among(N,X,\mathcal{V})$}
  \label{fig:among}
\end{figure}
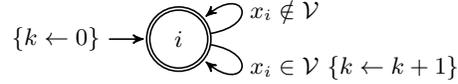

If each signature variable depends on a sliding window of size $1$
within $X$ (as for $\Among$), then the signature constraints are
\emph{unary}.  Our results also apply to cDFAs with unary signature
constraints because a network of a $\cAutomatonAtMost$ constraint and
unary signature constraints is Berge-acyclic.

\section{The Propagator}
\label{section:propagator}

\subsection{Feasibility Test and Domain Consistency Filtering}

Our propagator is defined in terms of the following concepts, which
assume a sequence $x_1,\dots,x_n$ of variables:
\begin{itemize}
\item Define $\PreMin{i}$ (respectively $\PreMax{i}$) to be the set of
  pairs $\Tuple{q,c}$ where $c$ is the minimum (respectively maximum)
  counter increase (or value) after the automaton consumes string
  $\sigma$ from state $\Initial$ to reach state $q$, for all strings
  $\sigma = \sigma_1\cdots\sigma_i$ where $i \in [0,n]$ and $\sigma_j
  \in \Domain{x_j}$ for each $j \in [1,i]$.
\item Define $\SufMin{i}$ (respectively $\SufMax{i}$) to be the set of
  pairs $\Tuple{q,c}$ where $c$ is the minimum (respectively maximum)
  counter increase after the automaton consumes string $\sigma$ from
  state $q$ to reach a state appearing in $\PreMin{n}$ (respectively
  $\PreMax{n}$), for all strings $\sigma = \sigma_i\cdots\sigma_n$
  where $i \in [1,n+1]$ and $\sigma_j \in \Domain{x_j}$ for each $j
  \in [i,n]$.
\end{itemize}

\begin{example}
  By illustrating one representative of these four quantities, we show
  that we have to maintain the maximum counter value \emph{for every
    state} reachable from $\Initial$ in $i$ steps, rather than just
  maintaining the overall maximum counter value and the set of states
  reachable from $\Initial$ in $i$ steps.  Consider the automaton
  $\mathcal{RST}$ in Figure~\ref{fig:rev_automaton}, where $\Initial$
  is $\epsilon$.  In a sequence of $n=6$ variables $x_1,\dots,x_6$
  that must be assigned value `$\mathrm{r}$' or `$\mathrm{t}$', we
  have:
  \begin{align*}
    \PreMax{0} &= \Set{
      \Tuple{\epsilon,0}
    } \\
    \PreMax{1} &= \Set{
      \Tuple{\epsilon,0},
      \Tuple{r,1}
    } \\
    \PreMax{2} &= \Set{
      \Tuple{\epsilon,1},
      \Tuple{r,1},
      \Tuple{\mathit{rr},1}
    } \\
    \PreMax{3} &= \Set{
      \Tuple{\epsilon,1},
      \Tuple{r,2},
      \Tuple{\mathit{rr},1},
      \Tuple{\mathit{rrt},1}
    } \\
    \PreMax{4} &= \Set{
      \Tuple{\epsilon,2},
      \Tuple{r,2},
      \Tuple{\mathit{rr},2},
      \Tuple{\mathit{rrt},1},
      \Tuple{\mathit{rrtr},3}
    } \\
    \PreMax{5} &= \Set{
      \Tuple{\epsilon,2},
      \Tuple{r,3},
      \Tuple{\mathit{rr},3},
      \Tuple{\mathit{rrt},2},
      \Tuple{\mathit{rrtr},3}
    } \\
    \PreMax{6} &= \Set{
      \Tuple{\epsilon,3},
      \Tuple{r,3},
      \Tuple{\mathit{rr},3},
      \Tuple{\mathit{rrt},3},
      \Tuple{\mathit{rrtr},4}
    }
  \end{align*}
  Indeed, $\Tuple{\mathit{rrtr},4} \in \PreMax{6}$ because
  $\Tuple{\mathit{rrt},2} \in \PreMax{5}$ and there is a transition in
  $\DFA$ from $\mathit{rrt}$ to $\mathit{rrtr}$ on symbol
  `$\mathrm{r}$' with a counter increase of $2$, even though three
  states have a higher counter value (namely $3$) than $\mathit{rrt}$
  in $\PreMax{5}$.  \QED
\end{example}

\begin{figure}[t]
  \centering
  \begin{tikzpicture}[->,>=stealth',shorten >=1pt,auto,node distance=27mm,semithick]
    \node[initial,initial text=\Eq,initial distance=5mm,accepting,state] (e) {$\epsilon$};
    \node[accepting,state] (r)    [above of=e]   {$r$};
    \node[accepting,state] (rr)   at ($(r)+(1.8cm,0)$)   {$\mathit{rr}$};
    \node[accepting,state] (rrt)  at ($(rr)+(3.5cm,0)$)  {$\mathit{rrt}$};
    \node[accepting,state] (rrs)  [below of=rrt] {$\mathit{rrs}$};
    \node[accepting,state] (rrtr) [below of=rr] {$\mathit{rrtr}$};
    \path
    (e)    edge[in=240,out=210,loop]	node[left]{s, t}	(e)
           edge[bend left] node[left=1,above,sloped]{\Arc{r}{$k \gets k+1$}}	(r)
    (r)    edge 		node{s, t}			(e)
    	   edge			node{r}				(rr)
    (rr)   edge[in=150,out=120,loop]	node[left]{r}		(rr)
    (rr)   edge			node{t}				(rrt)
    (rr.290)      edge[bend right]	node[below]{s}			(rrs.150)
    (rrt)  edge[in=150,out=120,loop]	node[left]{t}		(rrt)
    (rrt)  edge[bend left]	node{s}				(rrs)
    (rrt.south west)  edge[bend angle=20,bend right]	node[right]{\Arc{r}{$k \gets k+2$}}	(rrtr)
    (rrs)  edge 		node {\Arc{r}{$k \gets k+2$}\hspace*{0.3cm}} (rrtr)
           edge[bend left]	node{s, t}			(e)
    (rrtr) edge[loop left]	node{t}				(rrtr)
           edge			node{r}				(rr)
           edge			node[right]{s}			(r);
  \end{tikzpicture}
  \caption{Counter-DFA $\mathcal{RST}$ with non-unit counter increases}
  \label{fig:rev_automaton}
\end{figure}
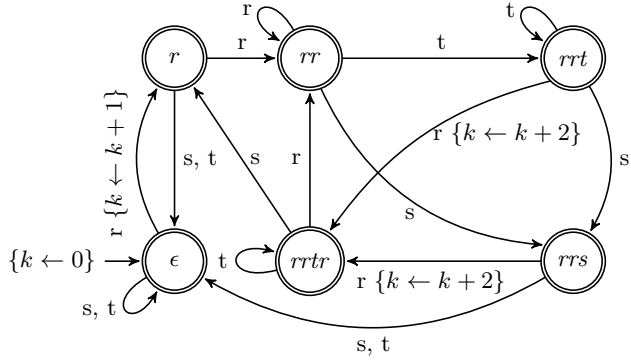

To compute $\PreMin{i}$ and $\SufMin{i}$, we need an operation that
takes a set of state-and-integer pairs and keeps only the pairs
$\Tuple{q,c}$ where there is no pair $\Tuple{q,c'}$ with $c' < c$.
Formally, $\TrimMin{}{S} = \SetComp{\Tuple{q,c} \in S}{\nexists
  \Tuple{q,c'} \in S : c' < c}$.  For brevity, we use
$\TrimMin{\phi(q,c)}{\Tuple{q,c}}$ to denote $\TrimMin{}{\left\{
    \Tuple{q,c} \mid \phi(q,c) \right\}}$, for any condition $\phi$.
We inductively define $\PreMin{i}$ and $\SufMin{i}$ as follows:
\[
\begin{array}{l}
  \PreMin{i} =
  \begin{cases}
    \Set{\Tuple{\Initial,0}}
    & \text{if~~} i=0 \\
    \TrimMin
    {\Tuple{q,c} \in \PreMin{i-1} \\ \ell \in \Domain{x_i}}
    {\Tuple{\NextQ{q}{\ell},~ c + \NextN{q}{\ell}}}
    & \text{if~~} i \in [1,n]
  \end{cases}
\end{array}
\]
\[
\begin{array}{l}
  \SufMin{i} =
  \begin{cases}
    \SetComp{\Tuple{q,0}}{\exists c \in \mathbb{N} : \Tuple{q,c} \in \PreMin{n}}
    & \text{if~~} i=n+1 \\
    \TrimMin
    {\Tuple{q',c'} \in \SufMin{i+1} \\
      \ell \in \Domain{x_i} \\
      \Next{q}{\ell} = \Tuple{q',\Inc}}
    {\Tuple{q,~ c' + \Inc}}
    & \text{if~~} i \in [1,n]
  \end{cases}
\end{array}
\]
We prove that the inductively computed quantities correspond to the
definitions of $\PreMin{i}$ and $\SufMin{i}$.  First consider
$\PreMin{i}$.  The base case $\PreMin{0}$ follows from the
initialisation to zero of the counter.  By induction, suppose the set
$\PreMin{i-1}$ is correct.  Before applying $\trimMin$, the set
contains all pairs obtained upon reading the symbol $\ell$ starting
from some pair $\Tuple{q,c}$ in $\PreMin{i-1}$, where $c$ is the
minimum counter value for $q$ over sequences of length $i-1$.  The
$\trimMin$ operation then filters out all the pairs $\Tuple{q',c'}$
with non-minimum counter value for $q'$.  The correctness proof for
$\SufMin{i}$ is similar.

We define the $\cAutomatonAtMost(N,X,\DFA)$ propagator.  The
propagator for $\cAutomatonAtLeast$ is similar.  The following theorem
gives a feasibility test.

\begin{theorem}\label{prop:feasibility}
  A $\cAutomatonAtMost(N,\Sequence{x_1,\dots,x_n},\DFA)$ constraint
  has a solution iff the minimum value of the counter of $\DFA$ after
  consuming the entire sequence is at most the maximum of the domain
  of $N$:
  \[
    \min_{\Tuple{q,c} \in \PreMin{n}} c
    \leq
    \max(\Domain{N})
  \]
\end{theorem}

\begin{proof}
  Suppose $\underline{c}$ is the minimum counter value such that
  $\Tuple{q,\underline{c}} \in \PreMin{n}$ for some state $q$.  By the
  definition of $\PreMin{n}$, there is some sequence $\sigma =
  \sigma_1\cdots\sigma_n$ where for all $1 \leq j \leq n$ the symbol
  $\sigma_j$ belongs to $\Domain{x_j}$ such that
  $\Cost{\Path{\Initial}{\sigma}{q}} = \underline{c}$.  Because each
  $\sigma_j$ belongs to the domain of the corresponding variable, we
  have that $\sigma$ is a solution to $\cAutomatonAtMost$ iff
  $\underline{c} \leq \max(\Domain{N})$.  \QED
\end{proof}

We now show how to achieve domain consistency on $\cAutomatonAtMost$.

\begin{theorem}\label{prop:filtering}
  For a $\cAutomatonAtMost(N,\Sequence{x_1,\dots,x_n},\DFA)$
  constraint, define the minimum value of the counter of $\DFA$ for
  variable $x_i$ to take value $\ell$:
  \[
    \Min{i,\ell} =
      \displaystyle
      \min_{\substack{
          \Tuple{q,c} \in \PreMin{i-1} \\
          q' = \NextQ{q}{\ell} \\
          \Tuple{q',c'} \in \SufMin{i+1}}}
      (c + \NextN{q}{\ell} + c')
  \]
  \begin{enumerate}
  \item A value $\ell$ in $\Domain{x_i}$ $(\text{with~} i \in [1,n])$
    appears in a solution iff the minimum value of the counter is at
    most the maximum of the domain of $N$:
    \begin{equation*}
      \label{prop:xi}
      \Min{i,\ell} \leq \max(\Domain{N})
    \end{equation*}
  \item A value in $\Domain{N}$ appears in a solution iff it is at
    least the minimum counter value given in
    Theorem~\ref{prop:feasibility}.
  \end{enumerate}
\end{theorem}

\begin{proof}
  We start with the first claim.
  (If) We show that any $\ell \in \Domain{x_i}$ with $\Min{i,\ell}
  \leq \max(\Domain{N})$ participates in a solution.  Suppose
  $\Min{i,\ell}$ equals $\underline{c} + \NextN{q}{\ell} +
  \underline{c}'$ for some $\Tuple{q,\underline{c}} \in \PreMin{i-1}$
  and some $\Tuple{q',\underline{c}'} \in \SufMin{i+1}$, with $q' =
  \NextQ{q}{\ell}$.  Then there exist two strings $\sigma =
  \sigma_1\cdots\sigma_{i-1}$ and $\tau = \sigma_{i+1}\cdots\sigma_n$
  and some state $q_n$ such that
  \[
    \Cost{\Path{\Initial}{\sigma}{q}} = \underline{c}
  \]
  and
  \[
    \Cost{\Path{q'}{\tau}{q_n}} = \underline{c}'
  \]
  with $\sigma_j \in \Domain{x_j}$ for all $j \in [1,n]$.  Note that
  the length of $\sigma\ell\tau$ is $n$.  We have:
  \begin{equation*}
    \begin{split}
      \Cost{\Path{\Initial}{\sigma\ell\tau}{q_n}} & = 
      \Cost{\Path{\Initial}{\sigma}{q}}
      + \NextN{q}{\ell}
      + \Cost{\Path{q'}{\tau}{q_n}} \\
      & = \underline{c} + \NextN{q}{\ell} + \underline{c}' \\
      & = \Min{i,\ell} \leq \max(\Domain{N}).
    \end{split}
  \end{equation*}
  Hence the assignment corresponding to $\sigma\ell\tau$ satisfies the
  domains and the constraint, so $\ell \in \Domain{x_i}$ participates
  in a solution.
  (Only~if) If $\ell \in \Domain{x_i}$ participates in a solution,
  then the counter of that solution is at least $\Min{i,\ell}$ and at
  most $\max(\Domain{N})$, hence $\Min{i,\ell} \leq \max(\Domain{N})$.

  The second claim follows from Theorem~\ref{prop:feasibility}.
  Indeed, let $\underline{c} = \min_{\Tuple{q,c} \in \PreMin{n}} c$ So
  there exists a sequence $\sigma = \sigma_1\cdots\sigma_n$ with each
  $\sigma_j \in \Domain{x_j}$ such that
  $\Cost{\Path{\Initial}{\sigma}{q}} = \underline{c}$.  Further, for
  any $\sigma' = \sigma'_1\cdots\sigma'_n$ with each $\sigma'_j \in
  \Domain{x_j}$ for all $j \in [1,n]$, we have $\underline{c} \leq
  \Cost{\Path{\Initial}{\sigma'}{q'_n}}$ for some state $q'_n$.
  So, by Theorem~\ref{prop:feasibility}, we need to prove that $v \in
  \Domain{N}$ participates in a solution iff $\underline{c} \leq v$.
  (Only if) If $v \in \Domain{N}$ participates in a solution, then
  there exists a sequence $\sigma' = \sigma'_1\cdots\sigma'_n$ such
  that each $\sigma'_j \in \Domain{x_j}$ and
  $\Cost{\Path{\Initial}{\sigma'}{q'_n}} \leq v$.  Since
  $\underline{c} \leq \Cost{\Path{\Initial}{\sigma'}{q'_n}}$, we have
  $\underline{c} \leq v$.
  (If) If $\underline{c} \leq v$, then the sequence $\sigma$ above
  necessarily also forms a solution with $N=v$.  \QED
\end{proof}

A propagator is obtained by directly implementing the expressions and
conditions of Theorems~\ref{prop:feasibility}
and~\ref{prop:filtering}.  It is idempotent.

\subsection{Complexity}

The complexity of a non-incremental implementation of the propagator
is established as follows.  Recall that we consider sequences of $n$
variables $x_i$, each with at most the automaton alphabet $\Sigma$ as
domain.  Let the automaton have $\Cardinality{Q}$ states.
Each set $\PreMin{i}$ has $O(\Cardinality{Q})$ elements and takes
$O(\Cardinality{\Sigma} \cdot \Cardinality{Q})$ time to construct and
trim (assuming it is implemented as a counter-value array indexed by
$Q$, with all cells initialised to $+\infty$).  There are $n+1$ such
sets, hence the entire $\PreMin{\cdot}$ vector takes $O(n \cdot
\Cardinality{\Sigma} \cdot \Cardinality{Q})$ time and $\Theta(n \cdot
\Cardinality{Q})$ space.
Similarly, the entire $\SufMin{\cdot}$ vector takes $O(n \cdot
\Cardinality{\Sigma} \cdot \Cardinality{Q})$ time and $\Theta(n \cdot
\Cardinality{Q})$ space.
Each value $\Min{i,\ell}$ takes $O(\Cardinality{Q})$ time to
construct, since at most $\Cardinality{Q}$ pairs $\Tuple{q,c}$ of
$\PreMin{i-1}$ are iterated over and the corresponding pair
$\Tuple{q',c'}$ is unique and can be retrieved in constant time (under
the assumed data structure).  There are $n \cdot \Cardinality{\Sigma}$
such values, hence the entire $\Min{\cdot,\cdot}$ matrix takes $O(n
\cdot \Cardinality{\Sigma} \cdot \Cardinality{Q})$ time and $\Theta(n
\cdot \Cardinality{\Sigma})$ space.
Each test of a domain value takes constant time, hence $\Theta(n+1)$
time in total for the $n$ variables $x_i$ and the counter variable
$N$.
In total, such an implementation takes $O(n \cdot \Cardinality{\Sigma}
\cdot \Cardinality{Q})$ time, and $\Theta(n \cdot (\Cardinality{Q} +
\Cardinality{\Sigma}))$ space.

\subsection{The Exact Regular Counting Constraint}

Not surprisingly, decomposing $\cAutomaton(N,X,\DFA)$ into the
conjunction of $\cAutomatonAtMost(N,X,\DFA)$ and
$\cAutomatonAtLeast(N,X,\DFA)$ does not yield domain consistency at
the fixpoint of their propagators: for the cDFA $\mathcal{B}$ in
Figure~\ref{fig:noDCexact} and the constraint
$\cAutomaton(N,\Sequence{2,x,2},\mathcal{B})$, with $N \in
\Set{0,1,2}$ and $x \in \Set{1,2}$, it misses the inference of $N \neq
1$.  Worse, achieving domain consistency on exact regular counting is
actually NP-hard:

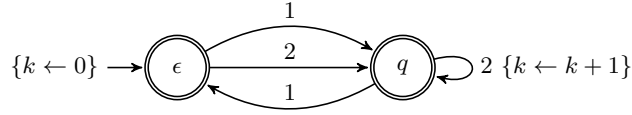
\begin{figure}[t]
  \centering
  \begin{tikzpicture}[->,>=stealth',shorten >=1pt,auto,node distance=30mm,semithick]
    \node[initial,accepting,initial text=\Eq,initial distance=5mm,state] (A) {$\epsilon$};
    \node[state,accepting] (B) [right of=A] {$q$};
    \path
    (A)
    edge[bend left]	node{1}					(B)
    edge		node{2}					(B)
    (B)
    edge[bend left]	node[above]{1}				(A)
    edge[loop right]	node[right]{\Arc{2}{$k \gets k+1$}}	(B);
  \end{tikzpicture}
  \caption{Counter-DFA $\mathcal{B}$, where domain consistency is not
    achieved for exact regular counting}
  \label{fig:noDCexact}
\end{figure}

\begin{theorem}\label{thm:exactDC:NPhard}
  The feasibility of $\cAutomaton$ constraints is NP-hard.
\end{theorem}

\begin{proof}
  By reduction from Subset-Sum.  Consider an instance
  $\Tuple{\Set{a_1,\dots,a_k},s}$ of Subset-Sum, which holds if there
  is a subset $A \subseteq \Set{a_1,\dots,a_k}$ such that $\sum_{v \in
    A} v = s$.  Construct a cDFA $\DFA$ with one state and alphabet
  $\Sigma = \Set{a_1,\dots,a_k,0}$.  A transition labelled by $a_i$
  increases the counter by $a_i$, and the transition labelled by $0$
  does not increase the counter.  Build a sequence of variables $X =
  \Tuple{x_1,\dots,x_k}$ such that $\Domain{x_i} = \Set{0,a_i}$ and a
  variable $N$ such that $\Domain{N} = \Set{s}$.  Such a reduction can
  be done in polynomial time.  Subset-Sum holds iff
  $\cAutomaton(N,X,\DFA)$ holds.  \QED
\end{proof}

The propagator for $\cAutomatonAtMost$ can be generalised into an
incomplete propagator for $\cAutomaton$.  A value $\ell$ is removed
from the domain of variable $x_i$ if the following condition holds:
\begin{gather*}
  \forall \Tuple{q,\underline{c}} \in \PreMin{i-1}: \\
  \displaystyle\bigwedge_{\substack{\Tuple{q,\overline{c}} \in \PreMax{i-1} \\
      q' = \NextQ{q}{\ell} \\
      \Tuple{q',\underline{c'}} \in \SufMin{i+1} \\
      \Tuple{q',\overline{c'}} \in \SufMax{i+1}
    }}
  \left[
    \hspace*{2pt}
    \underline{c} + \NextN{q}{\ell} + \underline{c'},~
    \overline{c} + \NextN{q}{\ell} + \overline{c'}
    \hspace*{2pt}
  \right]
  \cap \Domain{N} = \emptyset
\end{gather*}
This propagator has the same space complexity as $\cAutomatonAtMost$,
but it may need more than one run to achieve idempotency.  Indeed, it
differs from the previous propagator in that lower and upper bounds
have to be calculated for \emph{each} state in $\PreMin{i-1}$, and it
is possible that some states will give different bounds.  Hence the
first run of the propagator might not reach idempotency.  The
propagator is strictly stronger than computing the fixpoint of
$\cAutomatonAtMost$ and $\cAutomatonAtLeast$, because the intersection
test with respect to $\Domain{N}$ is strictly stronger than the
conjunction of the two comparisons on the \emph{atmost} and
\emph{atleast} sides: for the cDFA in Figure~\ref{fig:noDCexact} and
the $\cAutomaton(1,\Sequence{2,x,1,y,z},\mathcal{B})$ constraint, with
$x,y,z \in \Set{1,2}$, the $\cAutomaton$ propagator infers $z \neq 2$,
whereas the decomposition misses this inference.  The $\cAutomaton$
propagator is also incomplete: the counter-example before
Theorem~\ref{thm:exactDC:NPhard} for the decomposition also applies to
it.

\subsection{Evaluation}

We implemented in \emph{SICStus Prolog} version~4.2.1~\cite{SICStus}
the described propagators for $\cAutomatonAtMost$,
$\cAutomatonAtLeast$, and $\cAutomaton$.  As a sanity check, we tested
them extensively as follows.

We generated random cDFAs of up to five states (note that all $34$
counter automata of the \emph{Global Constraint
  Catalogue}~\cite{GC-catalogue} have at most five states) using the
random DFA generator~\cite{Almeida:randomDFA} of \emph{FAdo}
(version~0.9.6) and doing a counter increment by $1$ on each arc with
a probability of $20\%$.  For each random cDFA, we generated random
instances, with random lengths (up to $n=10$) of
$X=\Sequence{x_1,\dots,x_n}$ and random initial domains of the counter
variable $N$ (one value, two values, and intervals of length $2$ or
$3$) and the signature variables $s_i$ (intervals of any length, and
sets with holes).

The results, upon many millions of random instances, are that no
counterexample to the domain consistency of $\cAutomatonAtMost$ has
been generated (giving credence to Theorems~\ref{prop:feasibility}
and~\ref{prop:filtering}), and that no pruning by the three
propagators of actually supported values was observed.  Also, our
propagator for $\cAutomaton$ never propagates less but often more, to
the point of detecting more failures, than the built-in
$\Automaton$~\cite{Beldiceanu:automata} of SICStus Prolog, and that it
is already often up to twice faster than the latter, even though it is
currently na\"ively implemented in Prolog while the built-in works by
decomposition into a conjunction of other global constraints, all of
which are very carefully implemented in~\emph{C}.

Table~\ref{tab:eval} gives the cumulative runtimes (under Mac OS
X~10.7.5 on a 2.8~GHz Intel Core 2 Duo with a 4~GB RAM), the numbers
of detected failures, and (when both propagators succeed) the numbers
of pruned values for random instances of some constraints, the
four-state cDFA for the $\NumberWord(N,X,\text{``toto''})$ constraint
being unnecessary here.

\begin{table}[t]
  \begin{center}
    \begin{tabular}[c]{lr|rr|rr|rr}
      &
      & \multicolumn{2}{c|}{seconds}
      & \multicolumn{2}{c|}{failures}
      & \multicolumn{2}{c}{prunings} \\
      \cline{3-8}
      Constraint
      & \#inst
      & \Constraint{cAuto} & \Constraint{Auto}
      & \Constraint{cAuto} & \Constraint{Auto}
      & \Constraint{cAuto} & \Constraint{Auto} \\
      \hline
      $\Among(N,X,\mathcal{V})$
      &  4,400 & 0.8 & 1.8 & 2,241 & 2,241 &  3,749 & 3,749 \\
      $\NumberWord(N,X,\text{``aab''})$
      & 13,200 & 0.9 & 2.3 & 4,060 & 4,020 &  1,294 &   943 \\
      $\NumberWord(N,X,\text{``toto''})$
      & 17,600 & 0.9 & 2.7 & 4,446 & 4,435 &  1,149 &   663 \\
      $\cAutomaton(N,X,\mathcal{RST})$
      & 13,200 & 3.9 & 7.3 & 5,669 & 4,333 & 13,213 & 2,275 \\
      $\Constraint{inflexion}(N,X)$
      & 13,200 & 3.7 & 7.5 & 4,447 & 4,066 &  9,279 & 4,531 \\
    \end{tabular}
  \end{center}
  \caption{Comparison between $\cAutomaton$ and 
    $\Automaton$~\cite{Beldiceanu:automata} of SICStus Prolog}
  \label{tab:eval}
\end{table}

To demonstrate the power of our propagators, we have also tested them
on constraints whose counter-DFAs have \emph{binary} signature
constraints, so that our \emph{atleast} and \emph{atmost} regular
counting propagators may not achieve domain consistency, because they
were designed for unary signature constraints.  For example, the
$\Inflexion(N,X)$ constraint holds if there are $N$ inflexions (local
optima) in the integer sequence $X$; a cDFA is given
in~\cite{GC-catalogue}, with signature constraints using the
predicates $x_i ~\Set{<,=,>}~ x_{i+1}$ on the sliding window
$[x_i,x_{i+1}]$ of size $2$.  Our exact regular counting propagator
outperforms the built-in $\Automaton$~\cite{Beldiceanu:automata} of
SICStus Prolog, as shown in the last line of Table~\ref{tab:eval}.
Further, our instance generator has not yet constructed any
counter-example to domain consistency on \emph{atmost} regular
counting.

\section{Conclusion}
\label{section:conclusion}

This paper considers regular counting constraints over finite
sequences of variables, which are ubiquitous and very diverse in
sequencing and timetabling (e.g., restricting the number of monthly
working weekends, or two-day-periods where a nurse works during a
night followed by an afternoon).  It studies a class of counter
deterministic finite automata (cDFA) that provides much more concise
models for regular counting constraints than representations using
standard DFAs.

\subsection{Summary and Extensions}

Our main contribution is to show how to enforce domain consistency in
polynomial time for \emph{atmost} and \emph{atleast} regular
counting constraints, based on the frequent case of a counter-DFA with
only accepting states and a single counter that can be incremented by
transitions.  We have also proved that determining the feasibility of
\emph{exact} regular counting constraints is NP-hard.

It is possible to lift our restriction to counter automata where all
states are accepting, even though we are then technically outside the
realm of regular counting.  For instance, this would allow us to
constrain the number $N$ of occurrences of some pattern, recognised by
cDFA $\DFA_1$, in a sequence $X$ of variables, while $X$ is not
allowed to contain any occurrence of another pattern, recognised by
cDFA $\DFA_2$.  Rather than decomposing this constraint into the
conjunction of $\cAutomaton(N,X,\DFA_1)$ and
$\cAutomaton(0,X,\DFA_2)$, with poor propagation through the shared
variables, we can design a cDFA $\DFA_{12}$ that counts the number of
occurrences of the first pattern and \emph{fails} at any occurrence of
the second pattern (instead of \emph{counting} them) and post the
unique constraint $\cAutomaton(N,X,\DFA_{12})$, after using the
following recipe.  Add an accepting state, say $q$, and an alphabet
symbol, say $\$$, whose meaning is end-of-string.  Add transitions on
$\$$ from all existing accepting states to $q$, with counter increase
by zero.  Add transitions on $\$$ from all non-accepting states to
$q$, with counter increase by a suitably large number, such as
$\max(\Domain{N}) + 1$.  Make the non-accepting states accepting.
Append the symbol $\$$ to the sequence $X$ when posting the
constraint, so that the extended automaton never actually stops in a
state different from $q$, thereby making it irrelevant whether the
original states are accepting or not.

\subsection{Related Work}

Our regular counting constraints are related to the
$\CostRegular(X,\DFA,N,C)$ constraint~\cite{CostRegular}, an extension
of the $\Regular(X,\DFA)$ constraint~\cite{Pesant:regular}: a ground
instance holds if the sum of the variable-value assignment costs is
exactly $N$ after DFA $\DFA$ has accepted $X$, where the
two-dimensional cost matrix $C$, indexed by $\Sigma$ and $X$, gives
the costs of assigning each value of the alphabet $\Sigma$ of $\DFA$
to each variable of the sequence $X$.  Indeed, both the abstractions
and the underlying algorithms of regular counting constraints and
$\CostRegular$ are closely related.  However, we now argue that
regular counting constraints sometimes provide both a more natural
abstraction and some computational benefits, namely more propagation
and asymptotically less space, within the same asymptotic time.

At the \emph{conceptual level}, the regular counting and
$\CostRegular$ constraints differ in how costs are expressed.  In the
$\CostRegular$ constraint the costs are associated with variable-value
assignments, while in regular counting constraints the costs (seen as
counter increments) are associated with the transitions of the counter
automaton.  This is an important conceptual distinction, as counter
automata provide a more natural and compact abstraction for a variety
of constraints, where the focus is on counting rather than costing.
Footnote~1 of~\cite[page~318]{CostRegular} points out that the cost
matrix $C$ can be made three-dimensional, indexed also by the states
$Q$ of $\DFA$, but this is not discussed further
in~\cite{CostRegular}.  This allows the expression of costs on
transitions, and it seems that this has no impact on the time
complexity of their propagator.  This generalisation is implemented in
the \emph{Choco} solver~\cite{CHOCO}.  It is only with such a
three-dimensional cost matrix that it is possible for the modeller to
post a regular counting constraint by using the $\CostRegular$
constraint: first unroll the counter automaton for the length
$\Cardinality{X}$ into a directed acyclic weighted graph $G$ (as
described in~\cite{Pesant:regular}, and the counter increments become
the weights) and then post $\CostRegular(X,N,G)$; the \emph{Choco}
implementation~\cite[page~95]{CHOCO} of $\CostRegular$ features this
option.  The alternative is to read the three-dimensional cost matrix
$C$ off $\DFA$ only (since $X$ is not needed) and to post
$\CostRegular(X,\DFA',N,C)$, where DFA $\DFA'$ is counter-DFA $\DFA$
stripped of its counter increments.  Either way, this encoding is not
particularly convenient and it seems natural to adopt counter automata
as an abstraction.  Also note that the cost matrix $C$ has to be
computed for every different value of $\Cardinality{X}$ that occurs in
the problem model, while this is not the case with counter automata.
Essentially, regular counting is a specialisation of the generalised
$\CostRegular$ constraint (with a three-dimensional cost matrix),
obtained by projecting the generalised cost matrix onto two different
dimensions than in the original $\CostRegular$ constraint, namely $Q$
and $\Sigma$, and using it to extend the transition function of the
DFA to the signature $Q \times \Sigma \to Q \times \mathbb{N}$ and
calling the extended DFA a counter-DFA.

At the \emph{efficiency level}, regular counting constraints are
propagated using dynamic programming, like $\CostRegular$.  This is
not surprising.  The \emph{time} complexity is the same as for the
encoding using $\CostRegular$, namely $O(n \cdot \Cardinality{\Sigma}
\cdot \Cardinality{Q})$, where $n = \Cardinality{X}$, as the unrolling
of the cDFA takes the same time as the propagator itself.  It is
interesting however to note that the structure of regular counting
constraints enables a better \emph{space} complexity thanks to the
compactness of a counter automaton as the input datastructure, as well
as fundamentally different internal datastructures: we do \emph{not}
store the unrolled automaton.  Indeed, we have shown that regular
counting constraints have a space complexity of $\Theta(n \cdot
(\Cardinality{Q} + \Cardinality{\Sigma}))$, while the encoding by
$\CostRegular$ constraint takes $\Theta(n \cdot \Cardinality{\Sigma}
\cdot \Cardinality{Q})$ space, to store either the three-dimensional
cost matrix $C$ or the unrolled graph $G$ (\emph{Choco} allows both
ways of parametrising $\CostRegular$).  Note that adding a counter to
a DFA bears no asymptotic space overhead on the representation of the
DFA.

At the \emph{consistency level}, it is important to note that our
\emph{atmost} and \emph{atleast} regular counting propagators achieve
domain consistency on the counter variable $N$ in the same asymptotic
time as the $\CostRegular$ propagator~\cite{CostRegular,CHOCO}
achieves only bounds consistency on $N$.  As an aside, the claim
by~\cite{CostRegular,CHOCO} that their polynomial-time propagator
achieves domain consistency on the variables of the sequence $X$ is
invalidated by Theorem~\ref{thm:exactDC:NPhard}, hence this would only
be the case for \emph{atmost} and \emph{atleast} variants of
$\CostRegular$: for the cDFA $\mathcal{B}$ in
Figure~\ref{fig:noDCexact} and the constraint
$\cAutomaton(N,\Sequence{2,2,x,2,y},\mathcal{B})$, with $N \in
\Set{1,3}$ and $x,y \in \Set{1,2}$, their propagator misses the
inference of $y \neq 2$, and so does our propagator for exact regular
counting.  Our datastructures are more compact (see above), and yet
enable more propagation on $X$ for exact regular counting.

To summarise, although regular counting constraints and the
$\CostRegular$ constraint are closely related, we believe that the
results of this paper contribute both to our understanding of these
constraints and to the practice in the field.

The $\Constraint{SeqBin}$ constraint
\cite{Petit:seqBin,Katsirelos:seqBin} can be represented by a
regular counting constraint, but it would require non-unary
signature constraints.

\subsection{Future Work}

There are many issues that remain open.  They include the following
questions.  Can we implement our propagator to run in $O(n \cdot
\Cardinality{\Sigma})$ time?  Which counter-DFAs admit a propagator
achieving domain consistency for \emph{exact} regular counting?  Can
we generalise our domain-consistency result to non-unary signature
constraints?  Can we generalise our results to non-deterministic
counter automata (since all our notation depends on the transition
function $\TransitionFct$ being total)?

\subsection*{Acknowledgements}

The second and third authors are supported by grants 2011-6133 and
2012-4908 of the Swedish Research Council (VR).  NICTA is funded by
the Australian Government as represented by the Department of
Broadband, Communications and the Digital Economy and the Australian
Research Council through the ICT Centre of Excellence program.
We wish to thank Mats Carlsson for help with SICStus Prolog, Joseph
Scott for help with FAdo, and Arnaud Letord for help with performing
the test with $\CostRegular$ in \emph{Choco}.

\bibliographystyle{splncs03}
\bibliography{counterAuto}

\section*{Appendix: SICStus Prolog Implementation}

\begin{lstlisting}
:- use_module(library(lists)).
:- use_module(library(avl)).
:- use_module(library(clpfd)).

cautomatonAtmost(A, N, Vars) :-
	dom_suspensions(Vars, Susp),
	fd_global(cautomatonAtmost(A,N,Vars), atmost, [max(N)|Susp]).

cautomatonAtleast(A, N, Vars) :-
	dom_suspensions(Vars, Susp),
	fd_global(cautomatonAtleast(A,N,Vars), atleast, [min(N)|Susp]).

cautomaton(A, N, Vars) :-
	dom_suspensions(Vars, Susp),
	fd_global(cautomaton(A,N,Vars), among, [dom(N)|Susp], [idempotent(false)]).

dom_suspensions(Vs, Ss) :-
	(   foreach(V,Vs),
	    foreach(dom(V),Ss)
	do  true
	).

:- multifile clpfd:dispatch_global/4.
clpfd:dispatch_global(cautomatonAtmost(A,N,Vars), atmost, atmost, Actions) :-
	A = automaton(_, _, _, States, Transitions, _, _, _),
	trans_avl(Transitions, TransAVL),
	compute_qcf_qcb(States, TransAVL, N, Vars, Doms, QCF, QCB, MinN, _),
	fd_max(N, MaxN),
	MinN =< MaxN,
	Actions = [N in MinN..MaxN|Actions2],
	QCB = [_|RQCB],
	prune_vars(Vars, Doms, _, QCF, RQCB, atmost, TransAVL, _, MaxN, ValuesToRemove),
	prune_actions(ValuesToRemove, Actions2).

clpfd:dispatch_global(cautomatonAtleast(A,N,Vars), atleast, atleast, Actions) :-
	A = automaton(_, _, _, States, Transitions, _, _, _),
	trans_avl(Transitions, TransAVL),
	compute_qcf_qcb(States, TransAVL, N, Vars, Doms, QCF, QCB, _, MaxN),
	fd_min(N, MinN),
	MinN =< MaxN,
	Actions = [N in MinN..MaxN|Actions2],
	QCB = [_|RQCB],
	prune_vars(Vars, Doms, _, QCF, RQCB, atleast, TransAVL, MinN, _, ValuesToRemove),
	prune_actions(ValuesToRemove, Actions2).

clpfd:dispatch_global(cautomaton(A,N,Vars), among, among, Actions) :-
	A = automaton(_, _, _, States, Transitions, _, _, _),
	trans_avl(Transitions, TransAVL),
	compute_qcf_qcb(States, TransAVL, N, Vars, Doms, QCF, QCB, MinN, MaxN),
	MinN =< MaxN,
	Actions = [N in MinN..MaxN|Actions2],
	QCB = [_|RQCB],
	convert_dvars_to_doms([N], [DomN]),
	prune_vars(Vars, Doms, DomN, QCF, RQCB, among, TransAVL, MinN, MaxN, ValuesToRemove),
	prune_actions(ValuesToRemove, Actions2).

trans_avl(Transitions, TransAVL) :-
	(   foreach(Arc,Transitions),
	    fromto(KL1,[key(Q,V,0)-value(Q0,I),key(Q0,V,1)-value(Q,I)|KL2],KL2,[])
	do  (   Arc = arc(Q0,V,Q) -> I = 0
	    ;   Arc = arc(Q0,V,Q,[_+Inc]) -> I = Inc
	    )
	),
	keysort(KL1, KL3),
	keyclumped(KL3, KL4),
	ord_list_to_avl(KL4, TransAVL).

prune_actions(ValuesToRemove, Actions) :-
	(   foreach(X-Delete,ValuesToRemove),
	    foreach(X in_set KeepSet,Actions)
	do  fd_set(X, Set),
	    list_to_fdset(Delete, DelSet),
	    fdset_subtract(Set, DelSet, KeepSet)
	).

prune_vars(RV, RD, DomN, RF, RB, Flag, TransAVL, MinN, MaxN, Remove) :-
	(   foreach(Var,RV),
	    foreach(Dom,RD),
	    fromto(RF,[QCF_prev|RF1],RF1,[_]),
	    foreach(QCB_next,RB),
	    foreach(Var-ValuesToRemove,Remove),
	    param(DomN,Flag,TransAVL,MinN,MaxN)
	do  prune_var(Dom, DomN, Flag, TransAVL, QCF_prev, QCB_next, MinN, MaxN, ValuesToRemove)
	).

prune_var(R, DomN, Flag, TransAVL, QCF_prev, QCB_next, MinN, MaxN, S0) :-
	(   foreach(Val,R),
	    fromto(S0,S1,S2,[]),
	    param(DomN, Flag, TransAVL, QCF_prev, QCB_next, MinN, MaxN)
	do  (   check_value(QCF_prev, Val, DomN, Flag, TransAVL,  QCB_next, MinN, MaxN)
	    ->  S1 = [Val|S2]
	    ;   S1 = S2
	    )
	).

check_value(R, Val, DomN, Flag, TransAVL, QCB_next, MinN, MaxN) :-
	(   foreach(t(Q,MinQ,MaxQ),R),
	    param(Val, DomN, Flag, TransAVL, QCB_next, MinN, MaxN)
	do  check_value(Flag, Q, MinQ, MaxQ, Val, DomN, TransAVL, QCB_next, MinN, MaxN)
	).

check_value(atmost, Q, MinQ, _, Val, _, TransAVL, QCB_next, _, MaxN) :-
	(   get_next_prev_state(1, Q, Val, TransAVL, NextQ, Inc),
	    member(t(NextQ,MinNextQ,_), QCB_next)
	->  MinQ + Inc + MinNextQ > MaxN
	;   true
	).
check_value(atleast, Q, _, MaxQ, Val, _, TransAVL, QCB_next, MinN, _) :-
	(   get_next_prev_state(1, Q, Val, TransAVL, NextQ, Inc),
	    member(t(NextQ,_,MaxNextQ), QCB_next)
	->  MaxQ + Inc + MaxNextQ < MinN
	;   true
	).
check_value(among, Q, MinQ, MaxQ, Val, DomN, TransAVL, QCB_next, MinN, MaxN) :-
	(   get_next_prev_state(1, Q, Val, TransAVL, NextQ, Inc),
	    member(t(NextQ,MinNextQ,MaxNextQ), QCB_next)
	->  Inf is min(MaxN+1, MinQ + Inc + MinNextQ),
	    Sup is max(MinN-1, MaxQ + Inc + MaxNextQ)
	;   Inf = MaxN+1,
	    Sup = MinN-1
	),
	(   Inf =< Sup
	->  \+intersect(DomN, Inf, Sup)
	;   true
	).

compute_qcf_qcb(States, TransAVL, N, Vars, Doms, [First|RQCF], QCB, MinN, MaxN) :-
	convert_dvars_to_doms(Vars, Doms),
	once(member(source(InitialState), States)),
	First = [t(InitialState,0,0)],
	complete_qcf_qcb(Doms, 1, TransAVL, First, RQCF),
	last([First|RQCF], LastStates),
	fd_min(N, InitMinN), fd_max(N, InitMaxN),
	InitMin is InitMaxN + 1, InitMax is InitMinN - 1,
	set_counters_to_zero(LastStates, Last, InitMin, InitMax, MinN, MaxN),
	reverse(Doms, RDoms),
	complete_qcf_qcb(RDoms, 0, TransAVL, Last, RQCB),
	reverse([Last|RQCB], QCB).

complete_qcf_qcb(R, Flag, TransAVL, States0, S) :-
	(   foreach(Dom,R),
	    foreach(States2,S),
	    fromto(States0,States1,States2,_),
	    param(Flag,TransAVL)
	do  compute_nexts(States1, Flag, Dom, TransAVL, States2)
	).

compute_nexts(PrevStates, Flag, Dom, TransAVL, NextStates) :-
	compute_next_states(PrevStates, Flag, Dom, TransAVL, KL1, []),
	keysort(KL1, KL2),
	keyclumped(KL2, KL3),
	(   foreach(Tag-Pairs,KL3),
	    foreach(t(Tag,Min,Max),NextStates)
	do  transpose(Pairs,[MinOf,MaxOf]),
	    min_member(Min, MinOf),
	    max_member(Max, MaxOf)
	).

compute_next_states(R, Flag, Dom, TransAVL) -->
	(   foreach(t(Q,Min,Max),R),
	    param(Flag, Dom, TransAVL)
	do  (   foreach(Val,Dom),
		param(Flag,TransAVL,Q,Min,Max)
	    do  (   {get_next_prev_states(Flag, Q, Val, TransAVL, Values)},
		    (   foreach(value(NextPrevQ,Inc),Values),
			param(Min,Max)
		    do  {NextPrevMin is Min + Inc},
			{NextPrevMax is Max + Inc},
			[NextPrevQ-[NextPrevMin,NextPrevMax]]
		    )
		)
	    )
	).

get_next_prev_state(Flag, Q0, Val, TransAVL, Q, Inc) :-
	get_next_prev_states(Flag, Q0, Val, TransAVL, Values),
	member(value(Q,Inc), Values).

get_next_prev_states(Flag, Q0, Val, TransAVL, Values) :-
	avl_fetch(key(Q0,Val,Flag), TransAVL, Values), !.
get_next_prev_states(_, _, _, _, []).

set_counters_to_zero(R, S, Min0, Max0, Min, Max) :-
	(   foreach(t(Q,MinT,MaxT),R),
	    foreach(t(Q,0,0),S),
	    fromto(Min0,Min1,Min2,Min),
	    fromto(Max0,Max1,Max2,Max)
	do  Min2 is min(MinT, Min1),
	    Max2 is max(MaxT, Max1)
	).

convert_dvars_to_doms(R, S) :-
	(   foreach(V,R),
	    foreach(D,S)
	do  fd_set(V, SV),
	    fdset_to_list(SV, D)
	).

intersect([V|_], Inf, Sup) :-
	Inf =< V, V =< Sup, !.
intersect([_|R], Inf, Sup) :-
	intersect(R, Inf, Sup).
\end{lstlisting}

\end{document}